\documentclass[sigconf]{acmart}

\usepackage{booktabs} 

\usepackage{amsmath}
\usepackage{algorithm}
\usepackage{algpseudocode}
\usepackage{xspace}
\usepackage{url}

\usepackage{breakurl}
\usepackage{hhline}
\usepackage{afterpage}
\usepackage{enumerate}
\usepackage{graphicx}
\usepackage{tikz}
\usepackage{pgfplots}
\usepackage{subfigure}
\usepackage{array}
\usepackage{subfigmat}
\usepackage{etoolbox}
\usepackage{alphalph}
\usepackage{relsize}

\usepackage{multicol}
\usepackage{changepage}
\newenvironment{customlegend}[1][]{
	\begingroup
	\csname pgfplots@init@cleared@structures\endcsname
	\pgfplotsset{#1}
}{
	\csname pgfplots@createlegend\endcsname
	\endgroup
}
\pgfplotsset{every tick label/.append style={font=\Huge}}

\def\addlegendimage{\csname pgfplots@addlegendimage\endcsname}
\newcolumntype{P}[1]{>{\centering\arraybackslash}p{#1}}

\algnewcommand\algorithmicinput{\hspace*{\algorithmicindent}\textbf{Input}}
\algnewcommand\Input{\item[\algorithmicinput]}
\algnewcommand\algorithmicoutput{\hspace*{\algorithmicindent}\textbf{Output}}
\algnewcommand\Output{\item[\algorithmicoutput]}
\algnewcommand\Not{\textbf{not}\xspace}
\algnewcommand\In{\textbf{in}\xspace}
\algdef{SE}[FOR]{NoDoFor}{EndFor}[1]{\algorithmicfor\ #1}{\algorithmicend\ \algorithmicfor}%

\setcopyright{acmcopyright}

\acmConference[KDD WISDOM 2018]{7th KDD Workshop on Issues of Sentiment Discovery and Opinion Mining}{August 2018}{London, UK}

\begin{document}
\title{Fast Dawid-Skene: A Fast Vote Aggregation Scheme for Sentiment Classification}

\author{Vaibhav B Sinha, Sukrut Rao, Vineeth N Balasubramanian}
\affiliation{%
	\department{Department of Computer Science and Engineering}
	\institution{Indian Institute of Technology Hyderabad}
   \state{Telangana}
   \country{India}
   \postcode{502285}
}
\email{cs15btech11034@iith.ac.in, cs15btech11036@iith.ac.in, vineethnb@iith.ac.in}


\begin{abstract}
Many real world problems can now be effectively solved using supervised machine learning. A major roadblock is often the lack of an adequate quantity of labeled data for training. A possible solution is to assign the task of labeling data to a crowd, and then infer the true label using aggregation methods. A well-known approach for aggregation is the Dawid-Skene (DS) algorithm, which is based on the principle of Expectation-Maximization (EM). We propose a new simple, yet effective, EM-based algorithm, which can be interpreted as a `hard' version of DS, that allows much faster convergence while maintaining similar accuracy in aggregation. We show the use of this algorithm as a quick and effective technique for online, real-time sentiment annotation. We also prove that our algorithm converges to the estimated labels at a linear rate. Our experiments on standard datasets show a significant speedup in time taken for aggregation - upto $\sim$8x over Dawid-Skene and $\sim$6x over other fast EM methods, at competitive accuracy performance. The code for the implementation of the algorithms can be found at  \url{https://github.com/GoodDeeds/Fast-Dawid-Skene}.
\end{abstract}

%
%



\maketitle

\section{Introduction}
	Supervised learning has been highly effective in solving challenging tasks in sentiment analysis over the last few years. However, the success of supervised learning for the domain in recent years has been premised on the availability of large amounts of data to effectively train models. Obtaining a large labeled dataset is time-consuming, expensive, and sometimes infeasible; and this has often been the bottleneck in translating the success of machine learning models to newer problems in the domain.
	
	An approach that has been used to solve this problem is to crowdsource the annotation of data, and then aggregate the crowdsourced labels to obtain ground truths. Online platforms such as Amazon Mechanical Turk and CrowdFlower provide a friendly interface where data can be uploaded, and workers can annotate labels in return for a small payment. With the ever-growing need for large labeled datasets and the prohibitive costs of seeking experts to label large datasets, crowdsourcing has been used as a viable option for a variety of tasks, including sentiment scoring \cite{CSsentimentscoring}, opinion mining \cite{CScommodityreview}, general text processing \cite{Snow:2008:CFG:1613715.1613751}, taxonomy creation \cite{Bragg2013CrowdsourcingMC}, or domain-specific problems, such as in the biomedical field \cite{DBLP:journals/corr/GuanGDH17,Albarqouni2016AggNetDL}, among many others.

In recent times, there is a growing need for a fast and real-time solution for judging the sentiment of various kinds of data, such as speech, text articles, and social media posts. Given the ubiquitous use of the internet and social media today, and the wide reach of any information disseminated on these platforms, it is critical to have a efficient vetting process to ensure prevention of the usage of these platforms for anti-social and malicious activities. Sentiment data is one such parameter that could be used to identify potentially harmful content. A very useful source for identifying harmful content is other users of these internet services, that report such content to the service administrators. Often, these services are set up such that on receiving such a flag, they ask other users interacting with the same content to classify whether the content is harmful or not. Then, based on these votes, a final decision can be made, without the need for any human intervention. Some such works include: crowdsourcing the sentiment associated with words \cite{CSsentimenttoword}, crowdsourcing sentiment scoring for online media \cite{CSsentimentscoring}, crowdsourcing the classification of words to be used as a part of lexicon for sentiment analysis \cite{CSlexicon}, crowdsourcing sentiment judgment for video review \cite{CSvideoreview}, crowdsourcing for commodity review \cite{CScommodityreview}, and crowdsourcing for the production of word level annotation for opinion mining tasks \cite{CSsyntacticrelatedness}. However, with millions of users creating and adding new content every second, it is necessary that this decision be quick, so as to keep up with and effectively address all flags being raised. This indicates a need for fast vote aggregation schemes that can provide results for a stream of data in real time.

The use of crowdsourced annotations requires a check on the reliability of the workers and the accuracy of the annotations. While the platforms provide basic quality checks, it is still possible for workers to provide incorrect labels due to misunderstanding, ambiguity in the data, carelessness, lack of domain knowledge, or malicious intent. This can be countered by obtaining labels for the same question from a large number of annotators, and then aggregating their responses using an appropriate scheme. A simple approach is to use majority voting, where the answer which the majority of annotators choose is taken to be the true label, and is often effective. However, many other methods have been proposed that perform significantly better than majority voting, and these methods are summarized further in Section \ref{related}.
	
	Despite the various recent methods proposed, one of the most popular, robust and oft-used method to date for aggregating annotations is the Dawid-Skene algorithm, proposed by \cite{dawid1979maximum}, based on the Expectation Maximization (EM) algorithm. This method uses the M-step to compute error rates, which are the probabilities of a worker providing an incorrect class label to a question with a given true label, and the class marginals, which are the probabilities of a randomly selected question to have a particular true label.  These are then used to update the proposed set of true labels in the E-step, and the process continues till the algorithm converges on a proposed set of true labels (further described in Section \ref{dawidskenealgo}).
	
	In this work, we propose a new simple, yet effective, EM-based algorithm for aggregation of crowdsourced responses. Although formulated differently, the proposed algorithm can be interpreted as a `hard' version of Dawid-Skene (DS) \cite{dawid1979maximum}, similar to Classification EM \cite{celeux1992classification} being a hard version of the original EM. The proposed method converges upto 7.84x faster than DS, while maintaining similar accuracy. We also propose a hybrid approach, a combination of our algorithm with the Dawid-Skene algorithm, that combines the high rate of convergence of our algorithm and the better likelihood estimation of the Dawid-Skene algorithm as part of this work.
    
\section{Related Work}
	\label{related}
	The Expectation-Maximization algorithm for maximizing likelihood was first formalized by \cite{10.2307/2984875}. Soon after, Dawid and Skene \cite{dawid1979maximum} proposed an EM-based algorithm for estimating maximum likelihood of observer error rates, which became very popular for crowdsourced aggregation and is still considered by many as a baseline for performance. Many researchers, to this day, have worked on analyzing and extending the Dawid-Skene methodology (henceforth, called DS), of which we summarize the more recent efforts below. The work on crowdsourced data aggregation have not been confined only for sentiment analysis or opinion mining tasks, instead most of the methods are generic and can easily used for sentiment analysis and opinion mining tasks.
	
	A new model, GLAD, was proposed in \cite{NIPS2009_3644}, that could simultaneously infer the true label, the expertise of the worker, and the difficulty of the problem, and use this to improve on the labeling scheme. \cite{Raykar:2010:LC:1756006.1859894} improved upon DS by jointly learning the classifier while aggregating the crowdsourced labels. However, the efforts of \cite{NIPS2009_3644} were restricted to binary choice settings; and in the case of \cite{Raykar:2010:LC:1756006.1859894}, they focused on classification performance, which is however not the focus of this work. 
	
	\cite{ipeirotis2010quality} presented improvements over DS to recover from biases in labels provided by the crowd, such as cases where a worker always provides a higher label than the true label when labels are ordinal. More recently, \cite{NIPS2016_6124} analyzed and characterized the tradeoff between the cost of obtaining labels from a large group of people per data point, and the improved accuracy on doing so, as well as the differences in adaptive vs non-adaptive DS schemes.
	
	In addition to these efforts, there has also been a renewed interest in recent years to understand the rates of convergence of the Dawid-Skene method. \cite{minimax-optimal-convergence-rates-for-estimating-ground-truth-from-crowdsourced-labels} obtained the convergence rates of a projected EM algorithm under the homogeneous DS model, which however is a constrained version of the general DS model. \cite{NIPS2014_5431} proposed a two-stage algorithm which uses spectral methods to offset the limitations of DS to achieve near-optimal rate convergence. \cite{article} recently proposed a permutation-based generalization of the DS model, and derived optimal rates of convergence for these models. However, none of these efforts have explicitly focused on increasing the speed of convergence, or making Dawid-Skene more efficient in practice. The work in \cite{IWMV} is the closest in this regard, where they proposed an EM-based Iterative Weighted Majority Voting (IWMV) algorithm which experimentally leads to fast convergence. We use this method for comparison in our experiments.
	
	In addition to methods based on Dawid-Skene, other methods for vote aggregation have been developed, such as using Gaussian processes \cite{Rodrigues:2014:GPC:3044805.3044941} and online learning methods \cite{Welinder2010OnlineCR}. The scope of the problem addressed by Dawid-Skene has also been broadened, to allow cases such as when a data point may have multiple true labels \cite{DUAN20145723}. (In this work, we show how our method can be extended to this setting too.) For ensuring reliability of the aggregated label, a common approach is to use a large number of annotators, which may however increase the cost. To mitigate this, work has also been done to intelligently assign questions to particular annotators \cite{0768fc60fef84637864e13671a981243}, reduce the number of labels needed for the same accuracy \cite{Welinder2010OnlineCR}, consider the biases in annotators \cite{NIPS2011_4311} and so on. 
	Recent work on vote aggregation also includes deep learning-based approaches, such as \cite{Albarqouni2016AggNetDL,training-deep-neural-nets-aggregate-crowdsourced-responses,DBLP:journals/corr/abs-1709-01779}. A survey of many earlier methods related to vote aggregation can be found in the work of \cite{10.1007/978-3-642-41154-0_1} and \cite{sheshadri2013square}. Moreover, a benchmark collection of methods and datasets for vote aggregation is defined in \cite{sheshadri2013square}, which we use for evaluating the performance of our method.   
	
	While many new methods have been developed, the DS algorithm still remains relevant as being one of the most robust techniques, and is used as a baseline for nearly every new method. Inspired by \cite{celeux1992classification}, our work proposes a simple EM-based algorithm for vote aggregation, that provides a similar performance as Dawid-Skene but with a much faster convergence rate. We now describe our method.
    
	\section{Proposed Algorithm}
	\label{algos}
	We propose an Expectation-Maximization (EM) based algorithm for efficient vote aggregation. The E-step estimates the dataset annotation based on the current parameters, and the M-step estimates the parameters which maximize the likelihood of the dataset. Starting from a set of initial estimates, the algorithm alternates between the M-step and the E-step till the estimates converge. Although formulated using a different approach to the aggregation problem, we call our algorithm Fast Dawid-Skene (FDS), because of its similarity to the DS algorithm (described in Section \ref{dawidskenealgo}).
	
	\subsection{Preliminaries}
	\label{subsec_preliminaries}
	For convenience, we use the analogy of a question-answer setting to model the crowdsourcing of labels. The data shown to the crowd is viewed as a question, and the possible labels as choices of answers from the crowd worker/participant. Let the questions (data points, problems) that need to be answered be $q = \{1,2,3,\dots,Q\}$ and the annotators (participants, workers) labeling them be $a = \{1,2,3,\dots,A\}$. The task requires the participants to label each question by selecting one of the predefined set of choices (options),  $c = \{1,2,3,\dots, C\}$, which has the same length across all questions. A participant is said to answer a given question when s/he chooses an option as the answer for that question. A participant need not answer all the questions, and in fact, for a large pool of questions, it is reasonable to assume that a participant might be invited to answer only a small subset of all the questions. Each question is assumed to be answered by at least one participant (ideally, more). We also assume that the choice selected by a participant for a  question is independent of the choice selected by any other participant. This assumption holds for real-world applications that use contemporary crowdsourcing methods, where participants generally do not know each other, and are often physically and geographically separated, and thus do not influence each other. Besides, while answering a question, the participants have no knowledge of the choices chosen by previous participants in these settings. 
	
	\subsection{The Fast Dawid-Skene Algorithm}
	\label{ouralgo}
	We now derive the proposed Fast Dawid-Skene (FDS) algorithm under the assumption that each question has only one correct choice, and that a participant can select only one choice for each question. (In Section \ref{discussions}, we show how our method can be extended to relax this assumption.)
Our goal is to aggregate the choices of the crowd for a question and to approximate the correct choice. Consider the question $q$. Let the $K$ participants that answered this question be $\{q_1, q_2, \dots, q_K\}$.  The value of $K$ may vary for different questions. Let the choices chosen by these $K$ participants for question $q$ be $\{c_{q_1}, c_{q_2}, \dots, c_{q_K}\}$, and the correct (or aggregated) answer to be estimated for the question $q$ be $Y_q$. We define the answer to the question $q$ to be the choice $c \in \{1,2,\dots,C\}$ for which $P\left(Y_{q} = c | c_{q_1}, c_{q_2}, \dots, c_{q_K}\right)$ is maximum. Using Bayes' theorem and the independence assumption among participants' answers, we obtain:
	\begin{align}\label{e1}
	P&(Y_{q} = c | c_{q_1},c_{q_2},\dots, c_{q_K})\nonumber \\ 
	&= \frac{P(c_{q_1}, c_{q_2}, \dots, c_{q_K} | Y_{q} = c)P(Y_{q} = c)}{\sum\limits_{c=1}^{C} P(c_{q_1}, c_{q_2}, \dots, c_{q_K} | Y_{q} = c)P(Y_{q} = c)}\nonumber\\
	&= \frac{\left(\prod\limits_{k = 1}^{K} P(c_{q_k} | Y_{q} = c)\right)P(Y_{q} = c)}{\sum\limits_{c = 1}^{C} \left(\prod\limits_{k = 1}^{K} P(c_{q_k} | Y_{q} = c)\right)P(Y_{q} = c) }
	\end{align}
	
	Let $T_{qc}$ be the indicator that the answer to question $q$ is choice $c$. Using our formulation:
	\begin{equation}\label{e2}
	T_{qc} = \begin{cases} 1 &c = \underset{j \in \{1,2,\dots,C\}}{\arg\max} P(Y_{q} = c | c_{q_1}, c_{q_2}, \dots, c_{q_K}) \\ 0 & \text{otherwise}
	\end{cases}
	\end{equation}
	These $T_{qc}$s serve as the proposed answer sheet.
	
	To determine the correct (or aggregated) choice for a question $q$, we need the values of $P(c_{q_k} | Y_{q} = c)$ for all $k$ and $c$, which however is not known given only the choices from the crowd annotators. However, if the correct choices are known for all the questions, we can compute these parameters. Let $q_k$ be the annotator $a$. To compute the parameters, we first define the following sets:
	\begin{equation*}S_{a}^{(c)} = \left\{ i\, |\, Y_i = c \wedge a \text{ has answered question } i \right\}
	\end{equation*} 
	and 
	\begin{equation*}T_{c_a}^{(c)} = \left\{ i \,|\, Y_i = c \wedge a \text{ has answered } c_a \text{ on question } i \right\}
	\end{equation*} 
	Then, we have:
	\begin{equation}\label{e3}
	P(c_a | Y_{q} = c) = \frac{ \left| T_{c_a}^{(c)} \right|}{ \left| S_a^{(c)} \right|}
	\end{equation}
	where $\left| \cdot \right| $ denotes the cardinality of the set.
	Also, $P(Y_{q} = c)$ can be defined as:
	\begin{equation}\label{e4}
	P(Y_{q} = c) = \frac{\text{Number of questions having answer as }c}{\text{Total number of questions}}
	\end{equation}
	The above quantities can be estimated if we have the correct choices, and conversely, the correct choices can be obtained using the above quantities. We hence use an Expectation-Maximization (EM) strategy, where the E-step calculates the correct answer for each question, while the M-step determines the maximum likelihood parameters using equations \ref{e3} and \ref{e4}. There are no pre-calculated values of parameters to begin with, and so in the first E-step, we estimate the correct choices using majority voting. We continue applying the EM steps until convergence. We use the total difference between two consecutive class marginals being under a fixed threshold as the convergence criterion. We discuss the convergence criterion in more detail in Section \ref{experiments}. The proposed algorithm is summarized below in Algorithm \ref{fdsalgorithm}.
	\begin{algorithm}
		\caption{The Fast Dawid-Skene Algorithm}\label{fdsalgorithm}
		\begin{algorithmic}[1]
			\Input Crowdsourced choices of $Q$ questions by $A$ participants (annotators) from $C$ choices
			\Output Proposed true choices - $T_{qc}$
			\State Estimate $T$s using majority voting.
			\Repeat
			\State \textit{M-step:} Obtain the parameters, $P(c_a | Y_{q} = c)$ and $P(Y_{q} = c)$ using Equations \ref{e3} and \ref{e4}
			\State \textit{E-step:} Estimate $T$s using the parameters, $P(c_a | Y_{q} = c)$ and $P(Y_{q} = c)$, and with the help of Equations \ref{e2} and \ref{e1}.
			\Until convergence
		\end{algorithmic}
	\end{algorithm}
	
	\subsection{Connection to Dawid-Skene Algorithm}
	\label{dawidskenealgo}
	The Dawid-Skene algorithm \cite{dawid1979maximum} was one of the earliest EM-based methods for aggregation, and still remains popular and competitive to newer approaches. In this subsection, we briefly describe the Dawid-Skene methodology, and show the connection of our approach to this method.
	
	As defined in \cite{dawid1979maximum}, the maximum likelihood estimators for the DS method are given by:
	\begin{footnotesize}
		\begin{align*}
		\hat{\pi}_{cl}^{(a)} &= \frac{\text{number of times participant $a$ chooses $l$ when $c$ is correct}}{\text{number of questions seen by participant $a$ when $c$ is correct}}
		\end{align*}
	\end{footnotesize}
	\noindent and $\hat{p_c}$, which is the probability that a question drawn at random has a correct label of $c$. Let $n_{ql}^{(a)}$ be the number of times participant $a$ chooses $l$ for question $q$. Let $\{T_{qc} : q = 1,2,\dots, Q\}$ be the indicator variables for question $q$. If choice $m$ is true, for question $q$, $T_{qm} = 1$ and $\forall j \ne m,\,T_{qj} = 0$. Given the assumptions made in Section \ref{subsec_preliminaries}, when the true responses of all questions are available, the likelihood is given by:
	\begin{equation}\label{e8}
	\prod_{q=1}^{Q} \prod_{c=1}^{C} \left\{ p_c \prod_{a=1}^{A} \prod_{l=1}^{C} \left(\pi_{cl}^{(a)}\right)^{n_{ql}^{(a)}}\right\}^{T_{qc}}
	\end{equation}
	where $n_{ql}^{(a)}$ and $T_{qc}$ are known. Using equation \ref{e8}, we obtain the maximum likelihood estimators as:
	\begin{equation}\label{e9}
	\hat{\pi}_{cl}^{(a)} = \frac{\sum_q T_{qc} n_{ql}^{(a)}}{\sum_l \sum_q T_{qc} n_{ql}^{(a)}}
	\end{equation}
	\begin{equation}\label{e10}
	\hat{p}_c = \frac{\sum_q T_{qc}}{Q}
	\end{equation} 
	We then obtain using Bayes' theorem:
	\begin{equation}\label{e11}
	p(T_{qc} = 1 | \text{data}) = \frac{\prod_{a=1}^{A} \prod_{l=1}^{C} (\pi_{cl}^{(a)})^{n_{ql}^{(a)}} p_c }{ \sum_{r=1}^{C} \prod_{a=1}^{A} \prod_{l=1}^{C} (\pi_{rl}^{(a)})^{n_{ql}^{(a)}} p_r}
	\end{equation}
	The DS algorithm is then defined by using equations \ref{e9} and \ref{e10} to obtain the estimates of $p$s and $\pi$s in the M-step, followed by using equation \ref{e11} and the estimates of $p$s and $\pi$s to calculate the new estimates of $T$s in the E-step. These two steps are repeated until convergence (when the values don't change over an iteration).
	
	A close examination of the DS and proposed FDS algorithms shows that our algorithm can be perceived as a `hard' version of DS. The DS algorithm derives the likelihood assuming that the correct answers (which are ideally binary-valued) are known, but uses the values for $T_{qc}$ (which form a probability distribution over the choices) directly as obtained from equation \ref{e11}. Instead, in our formulation, we always have $T_{qc}$ as either $0$ or $1$ after each E-step. Our method is similar to the well-known Classification EM proposed in \cite{celeux1992classification}, which shows that a `hard' version of EM significantly helps fast convergence and helps scale to large datasets \cite{jollois2007speed}. We show empirically in Section \ref{experiments} that this subtle difference between DS and FDS ensures that changes in the answer sheet dampens down quickly, and allows our method to converge much faster than DS with comparable performance.
	A careful implementation for both FDS and DS provides a solution in $O(QACn)$ time under the assumption that there is only one correct choice for each question, where $n$ is the number of iterations required by the algorithm to converge. As the cost per iteration of FDS would be similar to DS by the nature of its formulation, this implies that the speedup of our algorithm is proportional to the ratio of the number of iterations required to converge by the two algorithms, which we also confirm experimentally.

	\subsection{Theoretical Guarantees for Convergence}
    In this subsection, we establish guarantees for convergence. We prove that if we start from an area close to a local maximum of the likelihood, we are guaranteed to converge to the maximum at a linear rate. For the analysis of our algorithm's convergence, we first frame it in a way similar to the Classification EM algorithm as proposed by \cite{celeux1992classification}. Classification EM introduces an extra C-step (Classification step) after the E-step. This is the step that assigns each question a single answer, thus doing a `hard' clustering of questions based on options instead of the `soft' clustering by DS.

	To continue with the proof we will use the notation used for DS. The term 
$ P(c_{q_k} | Y_{q} = c)$ for FDS is replaced by $\pi_{cc_{qk}}^{q_k}$ and the term 
$  P(Y_q = c) $ for FDS is replaced by $p_c$. $n_{ql}^{(a)}$ used by DS would be either $1$ or $0$ for the setting considered. 

Having established the analogy, we restate the algorithm in CEM form (Algorithm \ref{cemalgorithm}).

\begin{algorithm}
	\caption{The Fast Dawid-Skene Algorithm}\label{cemalgorithm}
	\begin{algorithmic}[1]
		\Input Crowdsourced choices of $Q$ questions by $A$ participants (annotators) from $C$ choices
		\Output Proposed true choices - $T_{qc}$
		\State Estimate $T$s using majority voting. This essentially does the first E and C step.
		\Repeat
		\State \textit{M-step:} Obtain the parameters, $\pi$s and $p$s using Equations \ref{e3} and \ref{e4}
		\State \textit{E-step:} Estimate $T$s using the parameters, $\pi$ and $p$, and with the help of Equation \ref{e1}.
		\State \textit{C-step:} Assign $T$s using the values obtained in the E-step and Equation \ref{e2}.
		\Until convergence
	\end{algorithmic}
\end{algorithm}

We prove the convergence of the CEM algorithm similar to \cite{celeux1992classification}. For the proof, let us first form partitions. We form $C$ partitions out of all the questions based on their correct answer in a step.
\begin{equation}
	P_c = \{q | Y_q = c\}
\end{equation}

In the CEM approach, each question can belong to only one partition. 
Now, we define the CML (Classification Maximum Likelihood) criterion:
\begin{equation}
	C_2(P,p,\pi) = \sum_{c=1}^{C} \sum_{q \in P_c} \log \left({ p_c f(q, \pi_c)}\right)
\end{equation}
In the above equation, $\pi_c = \{\pi_{cj}^{(a)} | \forall j \in \{1\dots C\} \text{ and a } \in \{1\dots A\} \}$ and 
\begin{equation}
	f(q,\pi_c) = \prod_{a=1}^{A} \prod_{l=1}^{C} \left(\pi_{cl}^{(a)}\right)^{n_{ql}^{(a)}}
\end{equation}

To prove convergence, we define a few more notations. Note that we begin the algorithm by first doing a majority vote. This assigns each question to a class and forms the first partition. We denote this partition as $P^0$. We then proceed to the M-step and estimate $\pi$ and $p$. Let us denote this first set of parameters by $\pi^1$ and $p^1$. The next EC step gives the next partition, $P^1$. Thus, the algorithm continues to calculate $(P^{m}, p^{m+1}, \pi^{m+1})$ from $(P^{m}, p^{m}, \pi^{m})$ in the M step. Then, in the EC step, it calculates $(P^{m+1}, p^{m+1}, \pi^{m+1})$ from $(P^{m}, p^{m+1}, \pi^{m+1})$.
\begin{theorem}
For the sequence $(P^{m}, p^{m}, \pi^{m})$ obtained by FDS, the value of $C_2(P^{m}, p^{m}, \pi^{m})$	increases and converges to a stationary value. Under the assumption that $p$s and $\pi$s are well defined, the sequence $(P^{m}, p^{m}, \pi^{m})$ converges to a stationary point.
\end{theorem}
\begin{proof}
	To prove the above theorem we prove that \\$C_2(P^{m+1}, p^{m+1}, \pi^{m+1}) \ge C_2(P^{m}, p^{m}, \pi^{m}) \, \forall m > 1$.\\
	Note that equations \ref{e3} and \ref{e4} maximize the likelihood given the values of $T$ and $n$ (as shown by \cite{dawid1979maximum}), i. e. $T$ is known, and so $\pi$s and $p$s obtained by the M-step maximize the likelihood. We need to show that maximizing the likelihood is the same as maximizing the CML criterion, $C_2$. In the case of hard clustering, for each $q$, only one class, $c$, can have $T_{qc}$ as $1$; all other classes will have $T_{qc}$ as 0. With this observation, we can rewrite the CML criterion as:
	\begin{align}
		C_2(P,p,\pi) &= \sum_{c=1}^{C} \sum_{q \in P_c} \log (p_c f(q, \pi_c))\\
					 &= \log \left\{\prod_{q=1}^{Q} \prod_{c=1}^{C} \left( p_c f(q, \pi_c) \right)^{T_{qc}} \right\}\\
					 &= \log \left\{ \prod_{q=1}^{Q} \prod_{c=1}^{C} \left( p_c \prod_{a=1}^{A} \prod_{l=1}^{C} \left(\pi_{cl}^{(a)}\right)^{n_{ql}^{(a)}} \right)^{T_{qc}} \right\}
	\end{align}
    
	Thus, maximizing maximum likelihood is equivalent to maximizing $C_2$. So, we have that after the M step, $C_2(P^{m}, p^{m+1}, \pi^{m+1}) \ge C_2(P^{m}, p^{m}, \pi^{m})$.\\
    
	Now, we consider the EC step. Observe that for each question $q$, we choose the answer as the option $c'$ for which $p_c' f(q,\pi_c') \ge p_c f(q,\pi_c)$ for all $c$ (By definition of the criterion for the C-step). Thus, $\log { p_c f(q, \pi_c)}$ increases individually for each question, and so cumulatively,
	$C_2(P^{m+1}, p^{m+1}, \pi^{m+1}) \ge C_2(P^{m}, p^{m+1}, \pi^{m+1})$.\\
	Combining the two inequalities, we obtain, 
	\begin{equation}
		C_2(P^{m+1}, p^{m+1}, \pi^{m+1}) \ge C_2(P^{m}, p^{m}, \pi^{m})
	\end{equation}
    
	This proves that $C_2$ increases at each step. Since the number of questions are finite and so the number of partitions as well are finite; the value of $C_2$ must converge after a finite number of iterations.\\
	On convergence, we obtain $ C_2(P^{m+1}, p^{m+1}, \pi^{m+1}) = \\C_2(P^{m}, p^{m+1}, \pi^{m+1}) = C_2(P^{m}, p^{m}, \pi^{m})$ for some $m$. By definition of the C-step, the first equality implies that $P^{m+1} = P^{m}$. Also under the assumption that $p$s and $\pi$s are well defined, we have that $p^m = p^{m+1}$ and $\pi^{m+1} = \pi^m$. This proves the convergence to a stationary point.
\end{proof}

To prove the rate of convergence, we define $M$ to be the set of matrices $U \in \mathbb{R}^{C \times Q}$ of nonnegative values. The matrices are defined such that the summation of values in each column is 1 and the summation along each row is nonzero.\\
Consider the criterion to be maximized as: 
\begin{equation}
	C_2'(U,p,\pi) = \sum_{c=1}^{C} \sum_{q=1}^{Q} u_{qc} \log (p_c f(q, \pi_c))
\end{equation}

With the above definitions, proposition 3 of \cite{celeux1992classification} guarantees a linear rate of convergence for FDS to a local maximum from a neighborhood around the maximum.

	\subsection{Hybrid Algorithm}
	\label{hybridalgo}
	While the proposed FDS method is quick and effective, by using the softer marginals, DS can obtain better likelihood values (which we found in some of our experiments too). A comparison of the likelihood values over multiple datasets (described in Section 4) is provided in Table 2.  To bring the best of both DS and FDS, we propose a hybrid version, where we begin with DS, and at each step, we keep track of sum of the absolute values of the difference in class marginals ($p_c$s). When this sum falls below a certain threshold, we switch to the FDS algorithm and continue (Algorithm \ref{hybalgorithm}). Our empirical studies showed that this hybrid algorithm can maintain high levels of accuracy along with faster convergence (Section \ref{experiments}). We however observe that a similar likelihood to DS does not necessarily translate to better accuracy, and in fact FDS outperforms Hybrid on some datasets.
	
	\begin{algorithm}
		\caption{The Hybrid Algorithm}\label{hybalgorithm}
		\begin{algorithmic}[1]
			\Input Crowdsourced choices for $Q$ questions by $A$ participants given $C$ choices per question, threshold $\gamma$
			\Output Aggregated choices: $T_{qc}$
			\State Estimate $T$s using majority voting.
			\Repeat
			\State \textit{M-step:} Obtain  parameters, $\hat{\pi}_{cl}^{(a)}$ and $\hat{p}_c$ using equations \ref{e9} and \ref{e10}
			\State \textit{E-step:} Estimate $T$s using parameters, $\hat{\pi}_{cl}^{(a)}$ and $\hat{p}_c$ using equation \ref{e11}.
			\Until $\sum_c | p_c^t - p_c^{t-1} | < \gamma$
			\Repeat
			\State EM steps of Algorithm \ref{fdsalgorithm} (FDS)
			\Until convergence
		\end{algorithmic}
	\end{algorithm}

	\section{Experimental Results}
	\label{experiments}
	We validated the proposed method on several publicly available datasets for vote aggregation, and the results are presented in this section. We first describe the datasets, competing methods used for comparison and the performance metrics used before presenting the results.
	
	\paragraph{Datasets:} We used seven real-world datasets to compare the performance of the proposed method against other methods. These include 
	\textit{LabelMe} \cite{Russell2008,R7807338}, \textit{SentimentPolarity (SP)} \cite{Pang:2005:SSE:1219840.1219855,Rodrigues:2014:GPC:3044805.3044941}, \textit{DAiSEE} \cite{d2016daisee,kamath2016crowdsourced}, and four datasets from the SQUARE benchmark \cite{sheshadri2013square}:  \textit{Adult2} \cite{ipeirotis2010quality}, \textit{BM} \cite{DBLP:journals/corr/abs-1209-3686}, \textit{TREC2010} \cite{Buckley10-notebook}, and \textit{RTE} \cite{Snow:2008:CFG:1613715.1613751}.
	
	Many of the datasets had varying number of annotators per data point. For uniformity, we set a threshold for each dataset, and all data points with fewer annotators than the threshold were removed. In our experiments, we studied the performance of all the methods by varying the number of annotators from one till the threshold, by taking a random subset of all annotators for a data point at each step (We maintained the same random seed across the methods, and conducted multiple trials to verify the results presented herewith). Also, the \textit{TREC2010} dataset has an `unknown' class, which we removed for our experiments. Table 1 lists the size, the number of classes, and the number of annotators in each dataset.
	\begin{table}[t]
		\begin{center}		
			\begin{scriptsize}
				\setlength\tabcolsep{3pt}
				\begin{tabular}{|P{1.1cm}||P{0.5cm}|P{0.8cm}|P{1cm}||P{1.1cm}|P{1.2cm}|P{1.1cm}|}
					\hline
					& \# qns &\# options (per qn)& Maximum \# of annotators (per qn) & Speedup of FDS over DS in Time (Iterations) & Speedup of FDS over IWMV in Time (Iterations) & Speedup of Hybrid over DS in Time (Iterations)\\
					\hhline{|=||=|=|=||=|=|=|}
					Adult2 & 305 & 4 & 9 & 6.61(7.87) & 1.32(1.15) & 2.30(2.43)\\ 
					\hline
					BM & 1000 & 2 & 5 & 2.69(4.51) & 1.70(1.02) & 1.49(2.03)\\ 
					\hline
					TREC2010 & 3670 & 4 & 5 & 7.84(8.64) & 6.09(2.93) & 4.39(4.59) \\
					\hline
					DAiSEE & 4628 & 4 & 10 & 6.57(7.37) & 4.40(2.04) & 4.11(4.37)\\
					\hline 
					LabelMe & 589 & 8 & 3 & 7.55(8.59) & 0.54(1.14) & 5.15(5.47)\\
					\hline
					RTE & 800 & 2 & 10 & 3.14(4.95) & 2.63(1.24) & 1.88(2.24)\\
					\hline
					SP & 4968 & 2 & 5 & 3.00(3.95) & 2.78(0.94) & 2.40(2.54)\\
					\hline
				\end{tabular}
				\captionof{table}{Datasets Used and Speedup of FDS and Hybrid}\label{datasettable}
			\end{scriptsize}
		\end{center}    
	\end{table}
	\vspace{-10pt}
	\paragraph{Baseline Methods:} A total of six aggregation algorithms were used in our experiments for evaluation - Majority Voting (MV), Dawid-Skene (DS) \cite{dawid1979maximum}, IWMV \cite{IWMV}, GLAD \cite{NIPS2009_3644}, proposed Fast Dawid-Skene (FDS), and the proposed hybrid algorithm. IWMV is among the fastest methods using EM for aggregation under general settings. \cite{IWMV} compared IWMV against other well-known aggregation methods, including \cite{Raykar:2010:LC:1756006.1859894}, \cite{Karger} and \cite{LPI}, and showed that IWMV gives an accuracy comparable to these algorithms but does so in a much lesser time. We hence compare our performance to IWMV in this work. GLAD \cite{NIPS2009_3644}, another popular method, was proposed only for questions with two choices, and we hence use this method for comparison only on the binary label datasets in our experiments. 
	
	\paragraph{Performance Metrics:} For each experiment, the following metrics were observed: the accuracy of the aggregated results (against provided ground truth), time taken and number of iterations needed for empirical convergence. For DS, FDS, and Hybrid, the negative log likelihood after each iteration was also observed. For MV, only the accuracy was observed. The experiments were conducted on a 4-core system with Intel Core i5-5200U 2.20GHz processors with 8GB RAM.
	
	\patchcmd{\subfigmatrix}{\hfill}{\hspace{0.01cm}}{}{}
	\begin{figure*}[t]
		\label{fig_result_graphs}
		\setlength\tabcolsep{0pt}
		\begin{tabular}{@{}ccccccc@{}}
			\centering
			\begin{footnotesize}Adult2\end{footnotesize} & \begin{footnotesize}BM\end{footnotesize} & \begin{footnotesize}TREC2010\end{footnotesize} & \begin{footnotesize}DAiSEE\end{footnotesize} & \begin{footnotesize}LabelMe\end{footnotesize} & \begin{footnotesize}RTE\end{footnotesize} & \begin{footnotesize}SP\end{footnotesize} \\
			\subfigure
			{
				\begin{tikzpicture}[scale=0.3]
				
				\begin{axis}[
				legend pos=outer north east,
				ymajorgrids=true,
				grid style=dashed,
				]
				
				\addplot[
				color=blue,
				mark=*,
				]
				coordinates {
					(1,0.7147540984)(2,0.7180327869)(3,0.7081967213)(4,0.7344262295)(5,0.7442622951)(6,0.7344262295)(7,0.7704918033)(8,0.7573770492)(9,0.7836065574)
				};
				\addplot[
				color=green,
				mark=square,
				]
				coordinates {
					(1,0.7147540984)(2,0.6950819672)(3,0.7409836066)(4,0.7704918033)(5,0.7442622951)(6,0.7508196721)(7,0.7540983607)(8,0.7540983607)(9,0.7573770492)
				};
				\addplot[
				color=red,
				mark=triangle,
				]
				coordinates {
					(1,0.7147540984)(2,0.7081967213)(3,0.7213114754)(4,0.7344262295)(5,0.7442622951)(6,0.7409836066)(7,0.7672131148)(8,0.7573770492)(9,0.7836065574)
				};
				\addplot[
				color=black,
				mark=x,
				]
				coordinates {
					(1,0.7147540984)(2,0.7049180328)(3,0.7540983607)(4,0.7573770492)(5,0.7409836066)(6,0.7508196721)(7,0.7540983607)(8,0.7442622951)(9,0.7508196721)
				};
				\addplot[
				color=purple,
				mark=o,
				]
				coordinates {
					(1,0.7147540984)(2,0.6819672131)(3,0.7508196721)(4,0.7672131148)(5,0.7442622951)(6,0.7508196721)(7,0.7475409836)(8,0.7540983607)(9,0.7508196721)
				};
				\end{axis}
				\end{tikzpicture}} & 
			\subfigure
			{
				\begin{tikzpicture}[scale=0.3]
				
				\begin{axis}[
				legend pos=outer north east,
				ymajorgrids=true,
				grid style=dashed,
				]
				
				\addplot[
				color=blue,
				mark=*,
				]
				coordinates {
					(1,0.685)(2,0.68)(3,0.699)(4,0.698)(5,0.712)
				};
				\addplot[
				color=green,
				mark=square,
				]
				coordinates {
					(1,0.685)(2,0.675)(3,0.679)(4,0.688)(5,0.697)
				};
				\addplot[
				color=red,
				mark=triangle,
				]
				coordinates {
					(1,0.685)(2,0.679)(3,0.694)(4,0.696)(5,0.709)
				};
				\addplot[
				color=black,
				mark=x,
				]
				coordinates {
					(1,0.685)(2,0.676)(3,0.686)(4,0.684)(5,0.696)
				};
				\addplot[
				color=purple,
				mark=o,
				]
				coordinates {
					(1,0.685)(2,0.685)(3,0.685)(4,0.694)(5,0.696)
				};
				\addplot[
				color=orange,
				mark=.,
				]
				coordinates {
					(1,0.685)(2,0.69)(3,0.685)(4,0.674)(5,0.695)
				};
				\end{axis}
				\end{tikzpicture}} & 
			\subfigure
			{
				\begin{tikzpicture}[scale=0.3]
				
				\begin{axis}[
				legend pos=outer north east,
				ymajorgrids=true,
				grid style=dashed,
				]
				
				\addplot[
				color=blue,
				mark=*,
				]
				coordinates {
					(1,0.3926430518)(2,0.4490463215)(3,0.483106267)(4,0.5359869139)(5,0.5638850889)
				};
				\addplot[
				color=green,
				mark=square,
				]
				coordinates {
					(1,0.3926430518)(2,0.4373297003)(3,0.4994550409)(4,0.5373500545)(5,0.5540355677)
				};
				\addplot[
				color=red,
				mark=triangle,
				]
				coordinates {
					(1,0.3926430518)(2,0.4435967302)(3,0.4811989101)(4,0.5436205016)(5,0.5707250342)
				};
				\addplot[
				color=black,
				mark=x,
				]
				coordinates {
					(1,0.3926430518)(2,0.4190735695)(3,0.489373297)(4,0.5215376227)(5,0.5529411765)
				};
				\addplot[
				color=purple,
				mark=o,
				]
				coordinates {
					(1,0.3926430518)(2,0.4144414169)(3,0.4678474114)(4,0.4991821156)(5,0.5212038304)
				};
				\end{axis}
				\end{tikzpicture}} & 
			\subfigure
			{
				\begin{tikzpicture}[scale=0.3]
				
				\begin{axis}[
				legend pos=outer north east,
				ymajorgrids=true,
				grid style=dashed,
				]
				
				\addplot[
				color=blue,
				mark=*,
				]
				coordinates {
					(1,0.3703543647)(2,0.4105445117)(3,0.3954191876)(4,0.5185825411)(5,0.514693172)(6,0.5131806396)(7,0.5432152118)(8,0.5246326707)(9,0.5386776145)(10,0.5311149525)
				};
				\addplot[
				color=green,
				mark=square,
				]
				coordinates {
					(1,0.3703543647)(2,0.4034140017)(3,0.4431719965)(4,0.4807692308)(5,0.4857389801)(6,0.5051858254)(7,0.5237683665)(8,0.5194468453)(9,0.5116681072)(10,0.4984874676)
				};
				\addplot[
				color=red,
				mark=triangle,
				]
				coordinates {
					(1,0.3703543647)(2,0.4170267934)(3,0.411840968)(4,0.5144770959)(5,0.5095073466)(6,0.5298184961)(7,0.5449438202)(8,0.5285220398)(9,0.5391097666)(10,0.5363007779)
				};
				\addplot[
				color=black,
				mark=x,
				]
				coordinates {
					(1,0.3703543647)(2,0.3900172861)(3,0.4019014693)(4,0.4572169404)(5,0.525713051)(6,0.5365168539)(7,0.5596369922)(8,0.5555315471)(9,0.5903197926)(10,0.5680639585)
				};
				\addplot[
				color=purple,
				mark=o,
				]
				coordinates {
					(1,0.3703543647)(2,0.3692739844)(3,0.4085998271)(4,0.426318064)(5,0.444900605)(6,0.4615384615)(7,0.4723422645)(8,0.4760155575)(9,0.4805531547)(10,0.4770959378)
				};
				\end{axis}
				\end{tikzpicture}} & 
			\subfigure
			{
				\begin{tikzpicture}[scale=0.3]
				
				\begin{axis}[
				legend pos=outer north east,
				ymajorgrids=true,
				grid style=dashed,
				]
				
				\addplot[
				color=blue,
				mark=*,
				]
				coordinates {
					(1,0.7623089983)(2,0.7640067912)(3,0.7877758913)
				};
				\addplot[
				color=green,
				mark=square,
				]
				coordinates {
					(1,0.7623089983)(2,0.7487266553)(3,0.7707979626)
				};
				\addplot[
				color=red,
				mark=triangle,
				]
				coordinates {
					(1,0.7623089983)(2,0.7674023769)(3,0.7894736842)
				};
				\addplot[
				color=black,
				mark=x,
				]
				coordinates {
					(1,0.7623089983)(2,0.7504244482)(3,0.765704584)
				};
				\addplot[
				color=purple,
				mark=o,
				]
				coordinates {
					(1,0.7623089983)(2,0.7215619694)(3,0.765704584)
				};
				\end{axis}
				\end{tikzpicture}} & 
			\subfigure
			{
				\begin{tikzpicture}[scale=0.3]
				
				\begin{axis}[
				legend pos=outer north east,
				ymajorgrids=true,
				grid style=dashed,
				]
				
				\addplot[
				color=blue,
				mark=*,
				]
				coordinates {
					(1,0.8)(2,0.79)(3,0.8725)(4,0.8925)(5,0.91125)(6,0.91125)(7,0.915)(8,0.91875)(9,0.9275)(10,0.9275)
				};
				\addplot[
				color=green,
				mark=square,
				]
				coordinates {
					(1,0.8)(2,0.8125)(3,0.8725)(4,0.8975)(5,0.915)(6,0.9225)(7,0.9125)(8,0.92375)(9,0.925)(10,0.91875)
				};
				\addplot[
				color=red,
				mark=triangle,
				]
				coordinates {
					(1,0.8)(2,0.79125)(3,0.88)(4,0.89125)(5,0.91125)(6,0.91125)(7,0.915)(8,0.92)(9,0.92875)(10,0.9275)
				};
				\addplot[
				color=black,
				mark=x,
				]
				coordinates {
					(1,0.8)(2,0.78125)(3,0.85875)(4,0.885)(5,0.9175)(6,0.92125)(7,0.91)(8,0.92625)(9,0.9275)(10,0.9275)
				};
				\addplot[
				color=purple,
				mark=o,
				]
				coordinates {
					(1,0.8)(2,0.7975)(3,0.85875)(4,0.86)(5,0.89375)(6,0.87875)(7,0.8825)(8,0.88375)(9,0.895)(10,0.895)
				};
				\addplot[
				color=orange,
				mark=.,
				]
				coordinates {
					(1,0.8)(2,0.86375)(3,0.85875)(4,0.9025)(5,0.91375)(6,0.91875)(7,0.91625)(8,0.9225)(9,0.93375)(10,0.9275)
				};
				\end{axis}
				\end{tikzpicture}} & 
			\subfigure
			{
				\begin{tikzpicture}[scale=0.3]
				
				\begin{axis}[
				legend pos=outer north east,
				ymajorgrids=true,
				grid style=dashed,
				]
				
				\addplot[
				color=blue,
				mark=*,
				]
				coordinates {
					(1,0.7904589372)(2,0.864331723)(3,0.8882850242)(4,0.9027777778)(5,0.9110305958)
				};
				\addplot[
				color=green,
				mark=square,
				]
				coordinates {
					(1,0.7904589372)(2,0.8558776167)(3,0.884057971)(4,0.9001610306)(5,0.9074074074)
				};
				\addplot[
				color=red,
				mark=triangle,
				]
				coordinates {
					(1,0.7904589372)(2,0.8655394525)(3,0.8868760064)(4,0.902979066)(5,0.9108293076)
				};
				\addplot[
				color=black,
				mark=x,
				]
				coordinates {
					(1,0.7904589372)(2,0.8586956522)(3,0.8772141707)(4,0.8963365539)(5,0.9023752013)
				};
				\addplot[
				color=purple,
				mark=o,
				]
				coordinates {
					(1,0.7904589372)(2,0.788647343)(3,0.8629227053)(4,0.866747182)(5,0.8824476651)
				};
				\addplot[
				color=orange,
				mark=.,
				]
				coordinates {
					(1,0.7904589372)(2,0.8297101449)(3,0.8627214171)(4,0.8822463768)(5,0.8913043478)
				};
				\end{axis}
				\end{tikzpicture}} \\
			\subfigure
			{
				\begin{tikzpicture}[scale=0.3]
				
				\begin{axis}[
				legend pos=outer north east,
				ymajorgrids=true,
				grid style=dashed,
				]
				
				\addplot[
				color=blue,
				mark=*,
				]
				coordinates {
					(1,0.0346820354)(2,0.3883323669)(3,0.8578050137)(4,0.5305502415)(5,0.6065449715)(6,1.1812446117)(7,0.9524474144)(8,0.668582201)(9,0.7890434265)
				};
				\addplot[
				color=green,
				mark=square,
				]
				coordinates {
					(1,0.0433580875)(2,0.0753462315)(3,0.1398756504)(4,0.0934870243)(5,0.0692028999)(6,0.1163470745)(7,0.1216783524)(8,0.1334779263)(9,0.1878612041)
				};
				\addplot[
				color=red,
				mark=triangle,
				]
				coordinates {
					(1,0.033028841)(2,0.1753511429)(3,0.2833549976)(4,0.3225502968)(5,0.3491220474)(6,0.3418936729)(7,0.3360276222)(8,0.4450104237)(9,0.3919889927)
				};
				\addplot[
				color=black,
				mark=x,
				]
				coordinates {
					(1,0.0355231762)(2,0.094111681)(3,0.1126391888)(4,0.1005189419)(5,0.1599042416)(6,0.1381061077)(7,0.2025053501)(8,0.1750547886)(9,0.1861321926)
				};
				\end{axis}
				\end{tikzpicture}} & 
			\subfigure
			{
				\begin{tikzpicture}[scale=0.3]
				
				\begin{axis}[
				legend pos=outer north east,
				ymajorgrids=true,
				grid style=dashed,
				]
				
				\addplot[
				color=blue,
				mark=*,
				]
				coordinates {
					(1,0.0593278408)(2,0.6203300953)(3,0.8350973129)(4,0.1861503124)(5,0.5905144215)
				};
				\addplot[
				color=green,
				mark=square,
				]
				coordinates {
					(1,0.0982289314)(2,0.2002208233)(3,0.2018971443)(4,0.1535782814)(5,0.2549500465)
				};
				\addplot[
				color=red,
				mark=triangle,
				]
				coordinates {
					(1,0.0551128387)(2,0.3184735775)(3,0.4568719864)(4,0.3007106781)(5,0.3789050579)
				};
				\addplot[
				color=black,
				mark=x,
				]
				coordinates {
					(1,0.1507251263)(2,0.3991253376)(3,0.249256134)(4,0.4386320114)(5,0.1799931526)
				};
				\addplot[
				color=orange,
				mark=.,
				]
				coordinates {
					(1,0.0935771465)(2,6.2674620152)(3,7.224709034)(4,2.9558873177)(5,2.7445700169)
				};
				\end{axis}
				\end{tikzpicture}} & 
			\subfigure
			{
				\begin{tikzpicture}[scale=0.3]
				
				\begin{axis}[
				legend pos=outer north east,
				ymajorgrids=true,
				grid style=dashed,
				]
				
				\addplot[
				color=blue,
				mark=*,
				]
				coordinates {
					(1,0.8667550087)(2,56.5786857605)(3,28.6649272442)(4,65.8875980377)(5,37.888890028)
				};
				\addplot[
				color=green,
				mark=square,
				]
				coordinates {
					(1,0.9969036579)(2,3.8249154091)(3,9.0065655708)(4,9.1470324993)(5,6.1260089874)
				};
				\addplot[
				color=red,
				mark=triangle,
				]
				coordinates {
					(1,0.8544027805)(2,10.2286441326)(3,10.6694591045)(4,11.0905911922)(5,11.1087253094)
				};
				\addplot[
				color=black,
				mark=x,
				]
				coordinates {
					(1,1.8852820396)(2,67.2027966976)(3,16.1378421783)(4,21.9081385136)(5,15.8872060776)
				};
				\end{axis}
				\end{tikzpicture}} & 
			\subfigure
			{
				\begin{tikzpicture}[scale=0.3]
				
				\begin{axis}[
				legend pos=outer north east,
				ymajorgrids=true,
				grid style=dashed,
				]
				
				\addplot[
				color=blue,
				mark=*,
				]
				coordinates {
					(1,0.6649360657)(2,40.8092784882)(3,48.3078978062)(4,70.405564785)(5,71.6696751118)(6,46.3514583111)(7,66.5292048454)(8,64.6717782021)(9,73.4134674072)(10,99.6033139229)
				};
				\addplot[
				color=green,
				mark=square,
				]
				coordinates {
					(1,0.8276641369)(2,3.574939251)(3,6.5229861736)(4,7.7969493866)(5,9.3732588291)(6,9.9729185104)(7,13.2241427898)(8,15.0871372223)(9,16.5652589798)(10,18.9662239552)
				};
				\addplot[
				color=red,
				mark=triangle,
				]
				coordinates {
					(1,0.6465461254)(2,5.5365962982)(3,13.8915581703)(4,13.7202095985)(5,14.4550738335)(6,14.9134073257)(7,16.4317111969)(8,22.5027289391)(9,23.401843071)(10,34.7314965725)
				};
				\addplot[
				color=black,
				mark=x,
				]
				coordinates {
					(1,1.0498731136)(2,14.9843654633)(3,14.1605768204)(4,25.0698599815)(5,34.7478704453)(6,83.9719145298)(7,53.7230751514)(8,51.0249986649)(9,100.392536879)(10,82.9250128269)
				};
				\end{axis}
				\end{tikzpicture}} & 
			\subfigure
			{
				\begin{tikzpicture}[scale=0.3]
				
				\begin{axis}[
				legend pos=outer north east,
				ymajorgrids=true,
				grid style=dashed,
				]
				
				\addplot[
				color=blue,
				mark=*,
				]
				coordinates {
					(1,0.1334607601)(2,4.028116703)(3,2.7392261028)
				};
				\addplot[
				color=green,
				mark=square,
				]
				coordinates {
					(1,0.1394340992)(2,0.5515549183)(3,0.3513114452)
				};
				\addplot[
				color=red,
				mark=triangle,
				]
				coordinates {
					(1,0.1195280552)(2,0.660323143)(3,0.6518347263)
				};
				\addplot[
				color=black,
				mark=x,
				]
				coordinates {
					(1,0.0688829422)(2,0.3313791752)(3,0.1715459824)
				};
				\end{axis}
				\end{tikzpicture}} & 
			\subfigure
			{
				\begin{tikzpicture}[scale=0.3]
				
				\begin{axis}[
				legend pos=outer north east,
				ymajorgrids=true,
				grid style=dashed,
				]
				
				\addplot[
				color=blue,
				mark=*,
				]
				coordinates {
					(1,0.0403387547)(2,0.5613822937)(3,0.8433725834)(4,1.0275821686)(5,0.407320261)(6,0.7001583576)(7,0.3837733269)(8,0.2267494202)(9,0.2961719036)(10,0.3385119438)
				};
				\addplot[
				color=green,
				mark=square,
				]
				coordinates {
					(1,0.0695331097)(2,0.1085121632)(3,0.1581466198)(4,0.1987450123)(5,0.2050540447)(6,0.1690099239)(7,0.2133939266)(8,0.1304950714)(9,0.1790511608)(10,0.2771937847)
				};
				\addplot[
				color=red,
				mark=triangle,
				]
				coordinates {
					(1,0.0382292271)(2,0.2137298584)(3,0.2925138474)(4,0.3072867393)(5,0.2757799625)(6,0.2854762077)(7,0.2920761108)(8,0.2746989727)(9,0.3147270679)(10,0.3276460171)
				};
				\addplot[
				color=black,
				mark=x,
				]
				coordinates {
					(1,0.075387001)(2,0.2511160374)(3,0.124822855)(4,0.2973849773)(5,0.5016331673)(6,0.3581020832)(7,0.6748847961)(8,0.7196099758)(9,0.5666649342)(10,0.7288057804)
				};
				\addplot[
				color=orange,
				mark=.,
				]
				coordinates {
					(1,0.0663247108)(2,7.5497150421)(3,3.1234180927)(4,2.4511742592)(5,5.304792881)(6,5.0181570053)(7,4.5251350403)(8,5.4739911556)(9,6.85029006)(10,3.9957911968)
				};
				\end{axis}
				\end{tikzpicture}} & 
			\subfigure
			{
				\begin{tikzpicture}[scale=0.3]
				
				\begin{axis}[
				legend pos=outer north east,
				ymajorgrids=true,
				grid style=dashed,
				]
				
				\addplot[
				color=blue,
				mark=*,
				]
				coordinates {
					(1,0.3821768761)(2,12.0007677078)(3,4.1331849098)(4,3.4475030899)(5,4.2648108006)
				};
				\addplot[
				color=green,
				mark=square,
				]
				coordinates {
					(1,0.5525462627)(2,2.5457980633)(3,1.7190337181)(4,1.4413399696)(5,1.7129518986)
				};
				\addplot[
				color=red,
				mark=triangle,
				]
				coordinates {
					(1,0.370223999)(2,2.5729827881)(3,2.6078674793)(4,2.309871912)(5,2.3132863045)
				};
				\addplot[
				color=black,
				mark=x,
				]
				coordinates {
					(1,1.5723991394)(2,5.7868909836)(3,5.9625639916)(4,3.4221427441)(5,5.1644203663)
				};
				\addplot[
				color=orange,
				mark=.,
				]
				coordinates {
					(1,0.7108900547)(2,8.9004328251)(3,8.3065419197)(4,8.8999679089)(5,30.9816091061)
				};
				\end{axis}
				\end{tikzpicture}} \\
			\setcounter{subfigure}{0}
			\subfigure
			{
				\begin{tikzpicture}[scale=0.3]
				
				\begin{axis}[
				legend pos=outer north east,
				ymajorgrids=true,
				grid style=dashed,
				]
				
				\addplot[
				color=blue,
				mark=*,
				]
				coordinates {
					(1,2)(2,19)(3,37)(4,21)(5,21)(6,37)(7,27)(8,17)(9,19)
				};
				\addplot[
				color=green,
				mark=square,
				]
				coordinates {
					(1,2)(2,3)(3,5)(4,3)(5,2)(6,3)(7,3)(8,3)(9,4)
				};
				\addplot[
				color=red,
				mark=triangle,
				]
				coordinates {
					(1,2)(2,8)(3,12)(4,12)(5,11)(6,10)(7,9)(8,11)(9,9)
				};
				\addplot[
				color=black,
				mark=x,
				]
				coordinates {
					(1,2)(2,4)(3,4)(4,3)(5,4)(6,3)(7,4)(8,3)(9,3)
				};
				\end{axis}
				\end{tikzpicture}} & 
			\subfigure
			{
				\begin{tikzpicture}[scale=0.3]
				
				\begin{axis}[
				legend pos=outer north east,
				ymajorgrids=true,
				grid style=dashed,
				]
				
				\addplot[
				color=blue,
				mark=*,
				]
				coordinates {
					(1,2)(2,21)(3,28)(4,6)(5,19)
				};
				\addplot[
				color=green,
				mark=square,
				]
				coordinates {
					(1,2)(2,4)(3,4)(4,3)(5,5)
				};
				\addplot[
				color=red,
				mark=triangle,
				]
				coordinates {
					(1,2)(2,8)(3,11)(4,7)(5,9)
				};
				\addplot[
				color=black,
				mark=x,
				]
				coordinates {
					(1,2)(2,5)(3,3)(4,5)(5,2)
				};
				\addplot[
				color=orange,
				mark=.,
				]
				coordinates {
					(1,1)(2,3)(3,3)(4,3)(5,2)
				};
				\end{axis}
				\end{tikzpicture}} & 
			\subfigure
			{
				\begin{tikzpicture}[scale=0.3]
				
				\begin{axis}[
				legend pos=outer north east,
				ymajorgrids=true,
				grid style=dashed,
				]
				
				\addplot[
				color=blue,
				mark=*,
				]
				coordinates {
					(1,2)(2,100)(3,42)(4,86)(5,46)
				};
				\addplot[
				color=green,
				mark=square,
				]
				coordinates {
					(1,2)(2,6)(3,12)(4,11)(5,7)
				};
				\addplot[
				color=red,
				mark=triangle,
				]
				coordinates {
					(1,2)(2,17)(3,15)(4,14)(5,13)
				};
				\addplot[
				color=black,
				mark=x,
				]
				coordinates {
					(1,2)(2,52)(3,10)(4,12)(5,8)
				};
				\end{axis}
				\end{tikzpicture}} & 
			\subfigure
			{
				\begin{tikzpicture}[scale=0.3]
				
				\begin{axis}[
				legend pos=outer north east,
				ymajorgrids=true,
				grid style=dashed,
				]
				
				\addplot[
				color=blue,
				mark=*,
				]
				coordinates {
					(1,2)(2,100)(3,87)(4,100)(5,83)(6,45)(7,54)(8,41)(9,42)(10,38)
				};
				\addplot[
				color=green,
				mark=square,
				]
				coordinates {
					(1,2)(2,7)(3,10)(4,10)(5,10)(6,9)(7,10)(8,9)(9,9)(10,7)
				};
				\addplot[
				color=red,
				mark=triangle,
				]
				coordinates {
					(1,2)(2,12)(3,23)(4,19)(5,15)(6,14)(7,13)(8,14)(9,13)(10,13)
				};
				\addplot[
				color=black,
				mark=x,
				]
				coordinates {
					(1,2)(2,19)(3,12)(4,16)(5,17)(6,33)(7,17)(8,13)(9,21)(10,14)
				};
				\end{axis}
				\end{tikzpicture}} & 
			\subfigure
			{
				\begin{tikzpicture}[scale=0.3]
				
				\begin{axis}[
				legend pos=outer north east,
				ymajorgrids=true,
				grid style=dashed,
				]
				
				\addplot[
				color=blue,
				mark=*,
				]
				coordinates {
					(1,2)(2,59)(3,35)
				};
				\addplot[
				color=green,
				mark=square,
				]
				coordinates {
					(1,2)(2,7)(3,4)
				};
				\addplot[
				color=red,
				mark=triangle,
				]
				coordinates {
					(1,2)(2,9)(3,8)
				};
				\addplot[
				color=black,
				mark=x,
				]
				coordinates {
					(1,2)(2,9)(3,4)
				};
				\end{axis}
				\end{tikzpicture}} & 
			\subfigure
			{
				\begin{tikzpicture}[scale=0.3]
				
				\begin{axis}[
				legend pos=outer north east,
				ymajorgrids=true,
				grid style=dashed,
				]
				
				\addplot[
				color=blue,
				mark=*,
				]
				coordinates {
					(1,2)(2,25)(3,35)(4,41)(5,15)(6,26)(7,14)(8,8)(9,10)(10,11)
				};
				\addplot[
				color=green,
				mark=square,
				]
				coordinates {
					(1,2)(2,3)(3,4)(4,5)(5,5)(6,4)(7,5)(8,3)(9,4)(10,6)
				};
				\addplot[
				color=red,
				mark=triangle,
				]
				coordinates {
					(1,2)(2,8)(3,10)(4,10)(5,9)(6,9)(7,9)(8,8)(9,9)(10,9)
				};
				\addplot[
				color=black,
				mark=x,
				]
				coordinates {
					(1,2)(2,5)(3,2)(4,4)(5,6)(6,4)(7,7)(8,7)(9,5)(10,6)
				};
				\addplot[
				color=orange,
				mark=.,
				]
				coordinates {
					(1,1)(2,3)(3,2)(4,3)(5,4)(6,4)(7,4)(8,3)(9,3)(10,2)
				};
				\end{axis}
				\end{tikzpicture}} & 
			\subfigure
			{
				\begin{tikzpicture}[scale=0.3]
				
				\begin{axis}[
				legend pos=outer north east,
				ymajorgrids=true,
				grid style=dashed,
				]
				
				\addplot[
				color=blue,
				mark=*,
				]
				coordinates {
					(1,2)(2,50)(3,21)(4,17)(5,20)
				};
				\addplot[
				color=green,
				mark=square,
				]
				coordinates {
					(1,2)(2,9)(3,6)(4,5)(5,6)
				};
				\addplot[
				color=red,
				mark=triangle,
				]
				coordinates {
					(1,2)(2,11)(3,11)(4,10)(5,10)
				};
				\addplot[
				color=black,
				mark=x,
				]
				coordinates {
					(1,2)(2,7)(3,7)(4,4)(5,6)
				};
				\addplot[
				color=orange,
				mark=.,
				]
				coordinates {
					(1,1)(2,2)(3,2)(4,2)(5,3)
				};
				\end{axis}
				\end{tikzpicture}} \\
			\begin{tikzpicture}
			\begin{customlegend}[legend columns=-1,
			legend style={
				draw=none,
				column sep=1ex,
			},legend entries={\begin{footnotesize}DS\end{footnotesize}}]
			\addlegendimage{blue,mark=*,sharp plot}
			\end{customlegend} 
			\end{tikzpicture} &
			\begin{tikzpicture}
			\begin{customlegend}[legend columns=-1,
			legend style={
				draw=none,
				column sep=1ex,
			},legend entries={\begin{footnotesize}IWMV\end{footnotesize}}]
			\addlegendimage{black,mark=x,sharp plot}
			\end{customlegend} 
			\end{tikzpicture} &
			\begin{tikzpicture}
			\begin{customlegend}[legend columns=-1,
			legend style={
				draw=none,
				column sep=1ex,
			},legend entries={\begin{footnotesize}MV\end{footnotesize}}]
			\addlegendimage{purple,mark=o,sharp plot}
			\end{customlegend} 
			\end{tikzpicture} &
			\begin{tikzpicture}
			\begin{customlegend}[legend columns=-1,
			legend style={
				draw=none,
				column sep=1ex,
			},legend entries={\begin{footnotesize}GLAD\end{footnotesize}}]
			\addlegendimage{orange,sharp plot}
			\end{customlegend} 
			\end{tikzpicture} &
			\begin{tikzpicture}
			\begin{customlegend}[legend columns=-1,
			legend style={
				draw=none,
				column sep=1ex,
			},legend entries={\begin{footnotesize}FDS\end{footnotesize}}]
			\addlegendimage{green,mark=square,sharp plot}
			\end{customlegend} 
			\end{tikzpicture} &
			\begin{tikzpicture}
			\begin{customlegend}[legend columns=-1,
			legend style={
				draw=none,
				column sep=1ex,
			},legend entries={\begin{footnotesize}Hybrid\end{footnotesize}}]
			\addlegendimage{red,mark=triangle,sharp plot}
			\end{customlegend} 
			\end{tikzpicture}
		\end{tabular}
		\caption{Experimental results: \textit{(Row 1:)} Accuracy of different methods across the considered datasets; \textit{(Row 2:)} Time taken in seconds to converge; and \textit{(Row 3:)} Number of iterations to converge. X-axis denotes the varying number of annotators studied for each dataset.}
	\end{figure*}
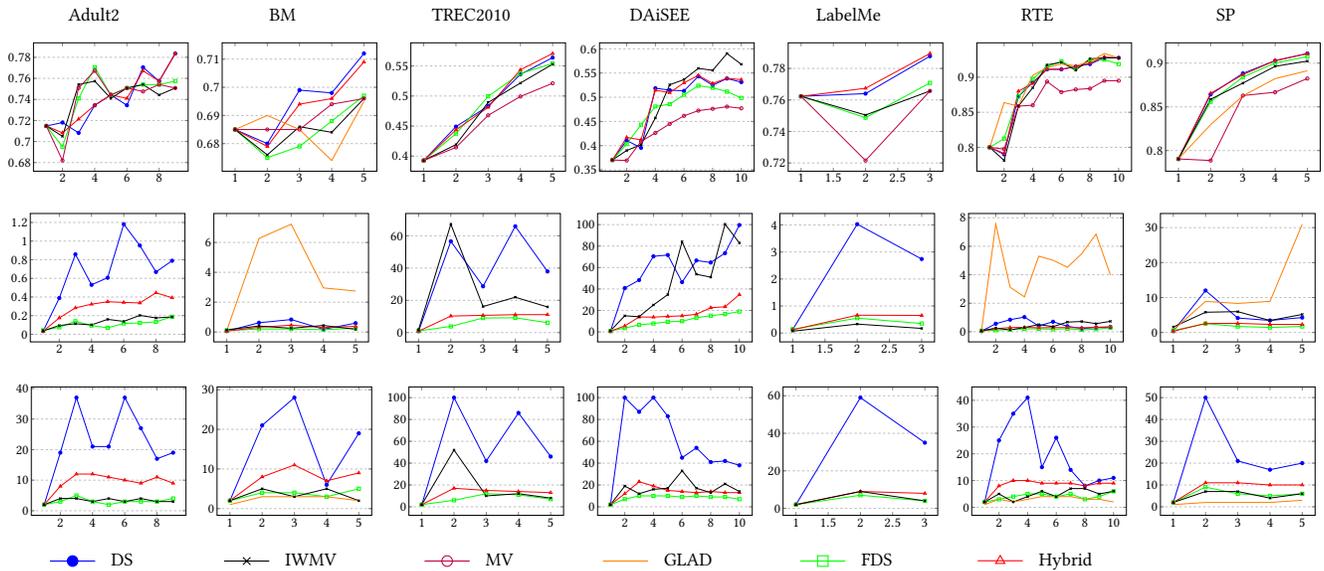
	
	\paragraph{Results:} The results of our experiments are presented in Figure 1 and Table \ref{logltable}. Table \ref{datasettable} shows the speedup in time and number of iterations needed to converge of FDS over DS and IWMV and of Hybrid over DS, averaged over all observations with varying number of annotators.

	\begin{table}[h]
		\begin{center}
			\begin{scriptsize}
				\begin{tabular}{ |c||c|c|c|}
					\hline
					& FDS & DS & Hybrid \\
					\hhline{|=||=|=|=|}
					Adult2 & 1283.75 & 1153.09 & 1154.97  \\ 
					\hline
					BM & 2110.16 & 2094.76 & 2100.32  \\ 
					\hline
					TREC2010 & 13109.26 & 12180.84 & 12346.91 \\
					\hline 
					DAiSEE & 39968.08 & 36178.16 & 36350.61 \\
					\hline
					LabelMe & 1714.50 & 1655.94 & 1660.06 \\
					\hline
					RTE & 3741.61 & 3679.63 & 3680.32 \\
					\hline
					SP & 12472.00 & 12433.70 & 12440.70 \\
					\hline
				\end{tabular}
			\end{scriptsize}
			\captionof{table}{Negative Log Likelihood at convergence of FDS, DS and Hybrid methods}\label{logltable}
		\end{center}
	\end{table}
	
	\paragraph{Performance Analysis of Fast Dawid-Skene:} The results show that FDS gives similar accuracies when compared to DS, Hybrid, GLAD, and IWMV, and a significant improvement over MV, on most datasets except for the BM and LabelMe datasets. In LabelMe, the aggregation accuracy is not at par with DS or Hybrid but is still significantly higher than MV and comparable to IWMV. In the BM dataset, the accuracies of FDS and IWMV are slightly lower than MV but both are comparable to each other. In terms of time taken, we notice that apart from the LabelMe dataset, FDS performs much better than DS, Hybrid, IWMV and GLAD all through. In the case of LabelMe, IWMV outperforms in terms of speed but the margin is very small (around 0.1 sec). This leads us to infer that in general, FDS gives comparable accuracies to other methods while taking significantly lesser time.
	
	\paragraph{Performance Analysis of the Hybrid Method:} The goal of the Hybrid algorithm is to converge to a similar likelihood as DS in much lesser time. From the experiments (especially Table 2), we see that this is indeed the case - the log likelihood of the Hybrid algorithm is close to that of DS and consistently better than FDS. This naturally leads to accuracies almost similar to those obtained by DS, as is confirmed in the results. The total time taken for convergence is much lower for Hybrid as compared to DS. Moreover, the time taken for convergence by Hybrid is consistently low and does not deviate as much as IWMV. While IWMV  outperforms Hybrid with respect to time in a few datasets, the proposed Hybrid outperforms IWMV on accuracy on those datasets. These observations support Hybrid to be an algorithm which performs with accuracies similar to DS in a much lesser time consistently over datasets.
	\paragraph{Implementation Details:}
	We discuss two important implementation details of the proposed methods in this section: \textit{initialization} and \textit{stopping conditions}. 
	As argued in \cite{dawid1979maximum}, a symmetric initialization of the parameters (all $P(Y_q = c)$s to be $1 / C$) corresponds to a start from a saddle point, from where the EM algorithm faces difficulty in converging. Instead, a good initialization is to start with the majority voting estimate. While performing majority voting, it could often happen that there is a tie between two or more options with the highest number of votes. In such situations, we randomly choose an option among those which received the highest votes\footnote{We also tried a variant, in which the option with the highest running class marginal was used to break ties.
		But this variant did not perform as well as the randomized majority voting across all methods. We also ran many trials with different random seeds, and found the results to almost the same as those presented.}. We maintained the same random seed for all methods which required this decision.
	
	The ideal convergence criterion would be when the answer sheet proposed by an algorithm stops changing. This condition is met within a few iterations for FDS and Hybrid, but DS does not converge using this criterion in a reasonable number of steps. For example, in case of the \textit{DAiSEE} dataset, DS did not converge even after 100 iterations (as compared to $\le 10$ for FDS). To address this issue,
	we set the convergence criterion as the point when the difference in class marginals is less than $10^{-4}$.
	We do not include the changes in participant error rates in the final convergence criterion because we observed that its fluctuations could lead to stopping prematurely. Similarly, the criterion for switching from DS to FDS in the Hybrid algorithm is the point when the change in class marginals is less than 0.005 (which happened approximately between 45-75\% of total iterations across the datasets).
\vspace{-5pt}
\section{Online Vote Aggregation}
Online aggregation of crowdsourced responses is an important setting in today's applications, where data points may be streaming in large data applications. 
	We consider a setting in which we have access to an initial set of questions and have obtained the proposed answer key using FDS. We also have $P(Y = c)$ and $P(c_a| Y = a) \,\forall\, c, a$ at this time. When we receive a new question and the answers from multiple participants for this new question, we first estimate the answer for this question directly using majority voting. We then update the parameters using the M-step in Algorithm \ref{fdsalgorithm}. After the M-step, we run the E-step only for this question to re-obtain the aggregated choice. To update the new knowledge which we have regarding the new participants, we run the M-step for one last time. We conducted experiments on the \textit{SP} dataset\footnote{More results, including on other datasets, on \url{https://sites.google.com/view/fast-dawid-skene/}},
	and observed almost the same accuracy for online FDS as offline FDS (Table 4) for different number of annotators. Table 3 shows the results for the max number of annotators (= 5).
	\begin{center}
		\label{onlinetable}
		\begin{scriptsize}
			\begin{tabular}{ |c||c|c|c|}
				\hline
				& DS & FDS & Hybrid\\
				\hhline{|=||=|=|=|}
				Accuracy & 90.94\% & 90.60\% & 90.64\%\\ 
				\hline
				Time taken to converge (s) & 4.40 & 3.76 & 4.09 \\ 
				\hline
				\# Iterations to converge & 26 & 4 & 5\\
				\hline 
			\end{tabular}
			\captionof{table}{Online Vote Aggregation on \textit{SP} dataset.
			}
		\end{scriptsize}
	\end{center}
    \vspace{-10pt}
	\begin{center}
		\label{onvsofftable}
		\begin{scriptsize}
			\begin{tabular}{ |c||c|c|c|c|}
				\hline
				Accuracy & 2 & 3 & 4 & 5\\
				\hhline{|=||=|=|=|=|}
				FDS & 85.59\% & 88.41\% & 90.02\% & 90.74\% \\ 
				\hline
				Online FDS & 83.57\% &  88.06\% & 89.90\% & 90.60\%\\ 
				\hline
			\end{tabular}
			\captionof{table}{Online FDS vs FDS for varying number of annotators.}
		\end{scriptsize}
	\end{center}
	
	\section{Extension to Multiple Correct Options}
	\label{discussions}
	The proposed FDS method can be extended to solve the aggregation problem under different settings. We describe an extension below, using the same notations as in Section \ref{subsec_preliminaries}.
	
	In real-world machine learning settings such as multi-label learning, a data point might belong to multiple classes, which would result in more than one true choice per question. For such cases, we now assume that participants are allowed to choose more than one choice for each question. Our Algorithm \ref{fdsalgorithm} originally assumes that every question has exactly one correct choice. To overcome this limitation, we can make a simple modification in how we interpret questions when multiple options are correct. We assume that every (question, option) pair is a separate binary classification problem, where the label is true if the option is chosen for that question, and false otherwise. This transforms a task with $Q$ questions and $C$ options each to a task with $QC$ questions and two options each. This is valid because the correctness of an option is independent of the correctness of all other options for that question in this setting. We ran experiments using this model on the Affect Annotation Love dataset \textit{(AffectAnnotation)} used in \cite{DUAN20145723} (which was specifically developed for this setting) on FDS, and compared our performance with DS and Hybrid. Our results are summarized in Table 5 (annotators=5, averaged over five subsets), showing the significantly improved results of FDS over DS. Hybrid attempts to follow DS in the likelihood estimation, and thus does not perform as well as FDS in this case. Besides, our results for FDS also performed better than the methods proposed in \cite{DUAN20145723}, which showed a best accuracy of $\approx92\%$ on this dataset.
	\begin{center}
		\label{multtable}
		\begin{scriptsize}
			\begin{tabular}{ |c||c|c|c|}
				\hline
				& DS & FDS & Hybrid \\
				\hhline{|=||=|=|=|}
				Accuracy & 88.66\% & 94.14\% & 89.26\%  \\ 
				\hline
				Time taken to converge (s) & 0.44 & 0.057 & 0.14 \\ 
				\hline
				\# Iterations to converge & 29.6 & 2 & 5.8 \\
				\hline 
			\end{tabular}
			\captionof{table}{Multiple Correct Options setting on \textit{AffectAnnotation} data.}
		\end{scriptsize}
	\end{center}
    \vspace{-20pt}
	\section{Conclusion}
	\label{conclusion}
	In this paper we introduced a new EM-based method for vote aggregation in crowdsourced data settings. Our method, Fast Dawid-Skene (FDS), turns out to be a `hard' version of the popular Dawid-Skene (DS) algorithm, and shows up to 7.84x speedup over DS and up to 6.09x speedup over IWMV in time taken for convergence. We also propose a hybrid variant that can switch between DS and FDS to provide the best in terms of accuracy and speed. We compared the performance of the proposed methods against other state-of-the-art EM algorithms including DS, IWMV and GLAD, and our results showed that FDS and the Hybrid approach indeed provide very fast convergence at comparable accuracies to DS, IWMV and GLAD. We proved that our algorithm converges to the estimated labels at a linear rate. We also showed how the proposed methods can be used for online vote aggregation, and extended to the setting where there are multiple correct answers, showing the generalizability of the methods.

\bibliographystyle{ACM-Reference-Format}
\bibliography{fastdawidskene}

\end{document}